\pdfoutput=1
\documentclass[11pt, a4paper]{article}

\usepackage[left=2.5cm, right=2.5cm, top=2.5cm, bottom=2.5cm]{geometry}
\usepackage{amsmath, amssymb, amsthm}
\usepackage{booktabs}
\usepackage{enumitem}
\usepackage[colorlinks=true, linkcolor=blue, citecolor=blue, urlcolor=cyan]{hyperref}

\hypersetup{
    pdftitle={The Variance of Thought: Policy Variance, Critical Forks, and Local Credit Assignment},
    pdfauthor={Yingru Li},
}

\newtheorem{theorem}{Theorem}
\newtheorem{proposition}[theorem]{Proposition}
\newtheorem{lemma}[theorem]{Lemma}
\newtheorem{corollary}[theorem]{Corollary}
\theoremstyle{definition}
\newtheorem{definition}[theorem]{Definition}

\newcommand{\E}{\mathbb{E}}
\newcommand{\Var}{\operatorname{Var}}
\newcommand{\Prob}{\mathbb{P}}
\newcommand{\Qpi}{Q^\pi}
\newcommand{\Vpi}{V^\pi}
\newcommand{\Api}{A^\pi}
\newcommand{\sigpi}{\sigma^2_\pi}
\newcommand{\Disp}{\mathcal{D}}

\author{Yingru Li\thanks{Email: \texttt{szrlee@gmail.com}}}
\title{\textbf{The Variance of Thought:\\ Policy Variance, Critical Forks,\\ and Local Credit Assignment}}
\date{First draft: July 14, 2025}

\begin{document}
\maketitle

\begin{abstract}
Long-horizon language-model tasks --- multi-step reasoning and tool-using agents alike --- are limited by credit assignment. We analyze it through the \emph{policy variance} $\sigpi(s)=\Var_{a\sim\pi}[\Qpi(s,a)]$, which in a deterministic MDP is the sole source of return variance and is injected in discrete pulses at states we call \emph{critical forks}. Three results follow. (i) Policy variance is a \emph{discovery budget}: observing an action of advantage $c$ requires $\Omega(c^2/\sigpi(s))$ draws, a bound that is exact on the canonical two-point fork. (ii) Policy variance is bounded by the policy's Gini dispersion, $\sigpi(s)\le 1-\|\pi(\cdot|s)\|_2^2$, a rollout-free necessary condition for criticality computable from logits alone. (iii) The remaining horizon sets the estimation cost: at a fork whose downstream success probability is $P$, the Monte Carlo advantage estimate has signal-to-noise ratio of order $\sqrt{P}$, so its sample cost scales as $1/P$ --- a cost that branched sampling shares. Bootstrapping removes it by converting a product of survival probabilities into a sum, provided the value representation is multiplicatively accurate, which argues for log-value parameterization.
\end{abstract}

\section{Introduction}

When a reasoning or agent trajectory of length $T$ receives a single terminal reward \cite{wei2022chain}, no token-level signal is directly observed. The standard framing treats the resulting gradient variance as noise to be suppressed \cite{williams1992simple,sutton2018reinforcement}. We adopt the complementary framing: in a deterministic environment, \emph{all} return variance originates in the policy's own choices, so the variance one wishes to suppress is exactly the statistic that measures how much any given choice matters. Suppressing it uniformly destroys the learning signal along with the noise.

Making this precise requires separating two questions:

\begin{enumerate}[label=(Q\arabic*), leftmargin=*]
\item \textbf{Local discovery.} At a fixed state $s$, how many action samples are needed to observe an action whose advantage exceeds $c$?
\item \textbf{Long-horizon estimation.} Once observed, how many trajectory samples are needed before its advantage is \emph{estimated} accurately enough to act on?
\end{enumerate}

After the setup of Section~\ref{sec:setup}, the two answers turn out to have different characters: (Q1) is governed by $\sigpi(s)$ and is \emph{polynomial}: at least $\Omega(c^2/\sigpi(s))$ samples, with the bound attained on the canonical fork (Section~\ref{sec:gating}), while (Q2) is governed by the remaining horizon and is \emph{exponential} (Section~\ref{sec:horizon}). The case for local credit assignment rests on (Q2). Section~\ref{sec:implications} develops the algorithmic consequences: a fork taxonomy, a detection procedure whose first stage is free, and the extension to stochastic environments.

\section{Setup: Policy Variance as the Sole Source of Instability}
\label{sec:setup}

We model reasoning as a finite-horizon MDP $(\mathcal S,\mathcal A,P,r,T)$ with $\gamma=1$, states $s_t$ the token prefix, actions $a_t$ the next token, and a single terminal reward $R\in[0,1]$.\footnote{Any bounded $R\in[R_{\min},R_{\max}]$ normalizes to $[0,1]$ affinely; policy gradients are invariant under affine reward transformations. All constants below are stated for the normalized scale.} Write $\Vpi(s)=\E_\pi[R\mid s]$, $\Qpi(s,a)=\E_\pi[R\mid s,a]$, $\Api(s,a)=\Qpi(s,a)-\Vpi(s)$, and $G_t$ for the Monte Carlo return. For pure reasoning the transition $s_{t+1}=s_t\oplus a_t$ is deterministic; Section~\ref{sec:implications} relaxes this. Determinism is used in exactly two places --- to collapse the variance decomposition below to its policy term, and to instantiate the chain of Section~\ref{sec:horizon}; the results of Section~\ref{sec:gating} (Lemma~\ref{lem:adv} through Corollary~\ref{cor:budget}) use only $\Qpi\in[0,1]$ and hold under arbitrary transition dynamics. Two elementary facts are used throughout: $\Qpi,\Vpi\in[0,1]$, hence $\Api(s,a)\in[-\Vpi(s),1-\Vpi(s)]$ and $|\Api|\le1$; and $\E_{a\sim\pi}[\Api(s,a)]=0$.

\begin{definition}[Policy variance]
\label{def:polvar}
\begin{equation*}
\sigpi(s)\;\triangleq\;\Var_{a\sim\pi(\cdot|s)}\big[\Qpi(s,a)\big]
\;=\;\Var_{a\sim\pi(\cdot|s)}\big[\Api(s,a)\big]
\;=\;\E_{a\sim\pi}\big[\Api(s,a)^2\big],
\end{equation*}
the equalities holding because $\Vpi(s)$ is constant at fixed $s$ and $\Api$ is centered.
\end{definition}

\begin{proposition}[Variance recursion]
\label{prop:recursion}
Let $\mathcal V_t\triangleq\E_{s_t\sim d^\pi_t}[\Var(G_t\mid s_t)]$. Then $\mathcal V_t=\E[C_t]+\gamma^2\mathcal V_{t+1}$ with
\begin{equation}
C_t(s_t)=\underbrace{\E_{a_t}\big[\Var_{s_{t+1}}\!\big(r_t+\gamma \Vpi(s_{t+1})\big)\big]}_{\text{environmental}}+\underbrace{\sigpi(s_t)}_{\text{policy}},
\label{eq:decomp}
\end{equation}
hence $\mathcal V_t=\sum_{k=t}^{T-1}\gamma^{2(k-t)}\E[C_k]$. Under deterministic transitions the environmental term vanishes identically, so with $\gamma=1$, $\mathcal V_t=\sum_{k\ge t}\E[\sigpi(s_k)]$.
\end{proposition}

\begin{proof}
Appendix~\ref{app:recursion}.
\end{proof}

In deterministic reasoning, return variance is therefore not a diffuse background but a sum of state-local pulses --- which motivates:

\begin{definition}[Critical fork]
\label{def:fork}
Fix $\tau>0$. A state $s$ is a \emph{$\tau$-critical fork} if $\sigpi(s)\ge\tau$.
\end{definition}

\section{Local Discovery Is Polynomially Gated}
\label{sec:gating}

\subsection{Advantage Is Gated by Action Probability}

\begin{lemma}[Advantage--probability bound]
\label{lem:adv}
For any $\pi$ and any $(s,a)$ with $\Qpi\in[0,1]$,
\begin{equation}
-\big(1-\pi(a|s)\big)\big(1-\Qpi(s,a)\big)\;\le\;\Api(s,a)\;\le\;\big(1-\pi(a|s)\big)\Qpi(s,a),
\label{eq:advbound}
\end{equation}
so in particular $|\Api(s,a)|\le 1-\pi(a|s)$.
\end{lemma}

\begin{proof}
Appendix~\ref{app:adv}.
\end{proof}

By~\eqref{eq:advbound}, high-advantage actions live in the tail of $\pi(\cdot|s)$: the search for improvement is a search over low-probability tokens. Combining the lemma with Definition~\ref{def:polvar} yields a bound on policy variance requiring no rollouts at all.

\begin{corollary}[Dispersion bound]
\label{cor:disp}
Let $\Disp(s)\triangleq 1-\|\pi(\cdot|s)\|_2^2$ denote the Gini--Simpson dispersion of the policy at $s$. Then
\begin{equation}
\sigpi(s)\;\le\;\sum_a\pi(a|s)\big(1-\pi(a|s)\big)^2\;\le\;\Disp(s),
\label{eq:disp}
\end{equation}
so if $\Disp(s)<\tau$ then $s$ is not a $\tau$-critical fork (Definition~\ref{def:fork}).
\end{corollary}

\begin{proof}
$\sigpi(s)=\sum_a\pi(a|s)\Api(s,a)^2\le\sum_a\pi(a|s)(1-\pi(a|s))^2$ by Lemma~\ref{lem:adv}, and $(1-p)^2\le 1-p$ on $[0,1]$.
\end{proof}

Condition~\eqref{eq:disp} costs $O(|\mathcal A|)$ time on the next-token distribution and no generation. On the two-point policy $(\varepsilon,1-\varepsilon)$ with value gap $\Delta$, $\sigpi=\varepsilon(1-\varepsilon)\Delta^2$ and $\Disp=2\varepsilon(1-\varepsilon)$, so the gap is a factor $2/\Delta^2$ --- within two of tight at maximal value gap, loose when values barely diverge. The condition is necessary but not sufficient: a policy dispersed over value-equivalent actions (paraphrases, formatting tokens) has $\Disp$ large and $\sigpi=0$. Dispersion screens; only rollouts or a critic confirm.

\subsection{Policy Variance Is a Discovery Budget}

We now answer (Q1) by bounding the tail of the single centered draw $\Api(s,\cdot)$, which has variance $\sigpi(s)$ and range in $[-1,1]$.

\begin{theorem}[Advantage sampling bound]
\label{thm:gate}
Fix $s$ and let $a\sim\pi(\cdot|s)$, $\sigma^2\triangleq\sigpi(s)$. For every $c>0$:
\begin{align}
\text{(one-sided)}\qquad &\Prob\big(\Api(s,a)\ge c\big)\;\le\;\frac{\sigma^2}{\sigma^2+c^2},\label{eq:cantelli}\\[2pt]
\text{(two-sided)}\qquad &\Prob\big(|\Api(s,a)|\ge c\big)\;\le\;\frac{\sigma^2}{c^2},\label{eq:chebyshev}\\[2pt]
\text{(success-rate)}\qquad &\Prob\big(\Api(s,a)\ge c\big)\;\le\;\frac{\min\{\Vpi(s),\,1-\Vpi(s)\}}{c}.\label{eq:markov}
\end{align}
\end{theorem}

\begin{proof}
Appendix~\ref{app:gate}: \eqref{eq:cantelli} is Cantelli's inequality, \eqref{eq:chebyshev} Chebyshev's, and \eqref{eq:markov} follows from $\E[(\Api)^+]=\E[(\Api)^-]\le\min\{\Vpi,1-\Vpi\}$ with Markov's inequality.
\end{proof}

\begin{proposition}[Exactness on the canonical fork]
\label{prop:tight}
Let $\pi(a_e|s)=\varepsilon$, $\pi(a_d|s)=1-\varepsilon$ with $\Qpi(s,a_e)=1$, $\Qpi(s,a_d)=0$, so $\sigma^2=\varepsilon(1-\varepsilon)$, $\Api(s,a_e)=1-\varepsilon$, $\Api(s,a_d)=-\varepsilon$. At $c=1-\varepsilon$, bound~\eqref{eq:cantelli} holds with equality:
$\sigma^2/(\sigma^2+c^2)=\varepsilon=\Prob(\Api\ge1-\varepsilon)$.
\end{proposition}

\begin{proof}
Direct computation; Appendix~\ref{app:gate} gives the general equality condition of Cantelli's inequality, of which this is an instance.
\end{proof}

The second-moment bounds are the sharp ones for a single draw; exponential-tail inequalities \cite{boucheron2013concentration} require a sample count to exponentiate, and Appendix~\ref{app:bernstein} quantifies their behavior in this setting. Restated as a sample requirement, Theorem~\ref{thm:gate} becomes a budget:

\begin{corollary}[Discovery budget]
\label{cor:budget}
Let $p_c(s)\triangleq\Prob(\Api(s,a)\ge c\mid s)$ and draw $a^{(1)},\dots,a^{(N)}\stackrel{iid}{\sim}\pi(\cdot|s)$. Then $1/p_c(s)\ge 1+c^2/\sigpi(s)$ and
\begin{equation}
\Prob\Big(\exists i:\Api(s,a^{(i)})\ge c\Big)\;\le\;N\,p_c(s)\;\le\;\frac{N}{1+c^2/\sigpi(s)},
\end{equation}
so discovery probability $\delta$ requires $N\ge\delta\big(1+c^2/\sigpi(s)\big)=\Omega\big(c^2/\sigpi(s)\big)$.
\end{corollary}

Policy variance is thus the reciprocal per-sample discovery rate: it \emph{is} the sampling budget, in rollouts --- tight by Proposition~\ref{prop:tight} and directly actionable as a branch count. On hard tasks the success-rate bound~\eqref{eq:markov} binds instead: with $\Vpi(s)=v\ll1$, the Bhatia--Davis inequality \cite{bhatia2000better} gives $\sigma^2\le v(1-v)$, so \eqref{eq:chebyshev} yields $\approx v/c^2$ while \eqref{eq:markov} yields the sharper $v/c$ for $c\le1$. On a problem the policy rarely solves, positive-advantage discoveries are rare in proportion to the success rate --- the formal version of the observation that RL from verifiable rewards \cite{ouyang2022training} produces almost no gradient signal at near-zero pass rate.

\section{Long-Horizon Estimation Is Exponentially Costly}
\label{sec:horizon}

Section~\ref{sec:gating} assumed exact $\Qpi$. In practice it is estimated from rollouts; this is (Q2), and it is where horizon enters. Consider the canonical structure: states $s_0,\dots,s_m$, a binary choice at each $s_j$, the wrong branch absorbing with reward $0$, and reward $1$ at $s_m$. With $p_j=\pi(\text{correct}\mid s_j)$ and $P_j\triangleq\prod_{k=j}^{m-1}p_k$ ($P_m=1$), Appendix~\ref{app:chain} gives
\begin{equation}
\Vpi(s_j)=P_j,\qquad
\Api(s_j,\text{correct})=(1-p_j)P_{j+1},\qquad
\sigpi(s_j)=p_j(1-p_j)P_{j+1}^2 .
\label{eq:chain}
\end{equation}
Note from~\eqref{eq:chain} that $|\Api(s_j,\text{correct})|\le 1-p_j$, consistent with Lemma~\ref{lem:adv}, with equality iff the suffix is deterministic; and that Corollary~\ref{cor:budget} at $c=(1-p_j)P_{j+1}$ evaluates to exactly $1/p_j$ draws --- the budget bound is attained on the chain as well.

\begin{proposition}[Monte Carlo estimation cost]
\label{prop:snr}
Condition on visiting $s_j$ and taking the correct action, and let $\hat A^{\mathrm{MC}}=G_j-\Vpi(s_j)$ with $G_j\sim\mathrm{Bernoulli}(P_{j+1})$ be the single-rollout advantage estimate with exact baseline. Then $\hat A^{\mathrm{MC}}$ is unbiased with
\begin{equation}
\mathrm{SNR}_j\;\triangleq\;\frac{\big|\Api(s_j,\mathrm{correct})\big|}{\sqrt{\Var(\hat A^{\mathrm{MC}})}}
\;=\;(1-p_j)\sqrt{\frac{P_{j+1}}{1-P_{j+1}}},
\end{equation}
so unit SNR requires averaging
\begin{equation}
n_j\;\asymp\;\frac{1}{(1-p_j)^2\,P_{j+1}}\;=\;\frac{1}{(1-p_j)^2}\exp\!\Big(\textstyle\sum_{k>j}\log(1/p_k)\Big)
\end{equation}
rollouts --- exponential in the remaining horizon $m-j$ whenever the $p_k$ are bounded away from $1$. The wrong branch, by contrast, yields $\hat A^{\mathrm{MC}}=-P_j$ deterministically: the failure signal is noiseless, and the statistical cost falls entirely on estimating the successful branch.
\end{proposition}

\begin{proof}
Appendix~\ref{app:chain}.
\end{proof}

The advantage and the value are of the same order $P_{j+1}$; the difficulty is that unbiased Bernoulli estimation of an exponentially small quantity costs its reciprocal. The cost is shared by every sampling scheme: branching at the fork estimates $\Qpi(s_j,\text{correct})=P_{j+1}$ by rollout from $s_{j+1}$ at the same $1/P_{j+1}$, so tree search \cite{yao2024tree} inherits it, and a whole-trajectory method faces the joint event that all $m$ forks resolve correctly, of probability $P_0=\exp(-\sum_k\log(1/p_k))$. A local estimator, by contrast, needs each fork resolved only marginally, conditional on reaching it.

\begin{proposition}[Log-space additivity]
\label{prop:log}
On the chain, $\Vpi(s_j)=p_j\Vpi(s_{j+1})$, hence
\begin{equation*}
\log \Vpi(s_j)=\log p_j+\log \Vpi(s_{j+1}),
\end{equation*}
so each bootstrapped log-space regression target is a single $O(1)$ increment and the log-value is a sum of $m-j$ of them. A critic fit in value space must instead resolve absolute differences of order $(1-p_j)P_{j+1}$; uniform absolute accuracy $\delta$ carries no information once $\delta\gtrsim(1-p_j)P_{j+1}$.
\end{proposition}

\begin{proof}
Appendix~\ref{app:chain}.
\end{proof}

Together, Propositions~\ref{prop:snr} and~\ref{prop:log} identify what long horizons require: a \emph{local credit resolver at forks} --- a learned critic \cite{schulman2015high,schulman2017proximal}, a value-equipped tree search \cite{silver2016mastering}, or any bootstrapped estimator --- and one that is multiplicatively calibrated (Proposition~\ref{prop:log}). A log-value critic meets both requirements; that it should dominate at long horizons is the framework's main empirical prediction.

\section{Algorithmic Implications}
\label{sec:implications}

\paragraph{Fork taxonomy.}
For binary forks with minority mass $\varepsilon\in(0,\tfrac12]$ and value gap $\Delta$, $\sigpi=\varepsilon(1-\varepsilon)\Delta^2$, and the pair $(\sigpi,H)$ of variance and policy entropy identifies $(\varepsilon,\Delta)$ uniquely (Appendix~\ref{app:taxonomy}); neither statistic alone suffices. Two regimes result. At a \textbf{Type I (uncertainty)} fork, $\varepsilon\approx\tfrac12$: variance is maximal and a value-divergent action is drawn within $O(1)$ samples, so the difficulty is \emph{decision} --- the appropriate intervention is a wider trust region permitting a decisive update. At a \textbf{Type II (exploration)} fork, $\varepsilon\ll1$: discovery requires $\Theta(1/\varepsilon)$ samples --- the lower bound is Corollary~\ref{cor:budget} at $c=1-\varepsilon$, and each draw hits the minority branch with probability $\varepsilon$, matching it --- and the difficulty is \emph{retention} of the rare positive update against the policy's confident default. The same $1/\varepsilon$ sample count reappears at the gradient level: the single-sample REINFORCE gradient at such a fork has SNR $\sqrt\varepsilon$ (Appendix~\ref{app:taxonomy}), so $\sigpi$ governs the gradient estimator as well as the return.

\paragraph{Amortized versus non-amortized resolution.}
The two implementations of a local credit resolver trade compute against bias:

\begin{center}
\begin{tabular}{@{}lll@{}}
\toprule
& \textbf{Learned critic $V_\phi$ (amortized)} & \textbf{Branching rollouts (non-amortized)}\\
\midrule
Per-fork cost & $O(1)$ forward pass & $\Omega(c^2/\sigpi)$ rollouts (Cor.~\ref{cor:budget})\\
Horizon cost & $O(1)$ via bootstrapping & $\Theta(1/P_{j+1})$ (Prop.~\ref{prop:snr})\\
Error source & approximation bias & sampling variance\\
Requirement & multiplicative accuracy (Prop.~\ref{prop:log}) & none\\
\bottomrule
\end{tabular}
\end{center}

The columns are complementary: branching rollouts supply the targets from which the critic is amortized, provided they are spent where the budget analysis says they are affordable.

\paragraph{Fork detection with a free first stage.}
(1)~\emph{Screen:} compute $\Disp(s_t)$ from the next-token distribution; if $\Disp(s_t)<\tau$, skip --- $s_t$ is not a $\tau$-critical fork by Corollary~\ref{cor:disp}. (2)~\emph{Budget:} for survivors, allocate $N\asymp c^2/\tau$ rollouts at detection scale $c$, guided by Corollary~\ref{cor:budget}, over $k$ candidate tokens drawn from both the top-$p$ head and the tail (covering both fork types), sharing the prefix KV cache. (3)~\emph{Allocate and estimate:} distribute rollouts by maximum standard error $\hat\sigma_i/\sqrt{N_i}$ (round-robin until $N_i\ge2$; Welford updates), then form $\widehat{\sigpi}(s_t)=\sum_i\pi_\theta(a_i|s_t)(\hat\mu_i-\bar\mu)^2$ and subtract its upward bias $\sum_i\pi_\theta(a_i)\hat\sigma_i^2/N_i$ --- with small $N_i$ the uncorrected estimator over-reports criticality exactly on the high-noise states where rollouts were scarce. (4)~\emph{Detect and classify:} flag $\widehat{\sigpi}$ as an outlier against an EWMA baseline, then classify by entropy: $H_t$ high $\Rightarrow$ Type I, low $\Rightarrow$ Type II.

\paragraph{Stochastic environments.}
For agents, the discovery-budget results of Section~\ref{sec:gating} apply unchanged, and the estimation cost of Proposition~\ref{prop:snr} only grows, since environmental variance adds to the estimator's noise. What changes is that the environmental term in~\eqref{eq:decomp} no longer vanishes, with three consequences. The dispersion screen bounds only the policy component, so it screens for policy-criticality alone. The TD advantage acquires residual variance $\Var_{s_{t+1}}(r_t+\gamma\Vpi(s_{t+1}))$ even under an exact critic --- a floor no critic accuracy removes. And attribution becomes ambiguous: an observed return deviation mixes the two terms of~\eqref{eq:decomp}, so the detector must separate them, requiring repeated rollouts from the same $(s_t,a_t)$ rather than from $s_t$ alone. This last point --- fork detection as a variance-\emph{decomposition} problem with uncharacterized sample cost --- is the open problem we consider most substantive.

\section{Conclusion}

Credit assignment in long-horizon reasoning is governed by two quantities of different character. Policy variance sets the local discovery budget --- $\Omega(c^2/\sigpi(s))$ samples, attained exactly on the two-point fork and bounded for free by the policy's Gini dispersion --- while the remaining horizon sets an exponential estimation cost shared by every sampling scheme, branched search included. Bootstrapping converts the underlying product of survival probabilities into a sum, and does so usefully under a multiplicatively accurate value representation. The framework yields a falsifiable prediction (log-value critics should dominate at long horizons), a free screening rule for fork detection, and a principled rollout budget in place of a hyperparameter.

\clearpage
\appendix

\section{Proof of Proposition~\ref{prop:recursion}}
\label{app:recursion}

\begin{proof}
Apply the law of total variance $\Var(X)=\E[\Var(X|Y)]+\Var(\E[X|Y])$ twice.

\emph{Step 1: condition on $a_t$.}
\[
\Var(G_t\mid s_t)=\E_{a_t}\big[\Var(G_t\mid s_t,a_t)\big]+\Var_{a_t}\big(\E[G_t\mid s_t,a_t]\big).
\]
Since $\E[G_t\mid s_t,a_t]=\Qpi(s_t,a_t)$, the second term is $\sigpi(s_t)$ by Definition~\ref{def:polvar}.

\emph{Step 2: condition on $s_{t+1}$.}
\[
\Var(G_t\mid s_t,a_t)=\E_{s_{t+1}}\big[\Var(G_t\mid s_t,a_t,s_{t+1})\big]+\Var_{s_{t+1}}\big(\E[G_t\mid s_t,a_t,s_{t+1}]\big).
\]
With $G_t=r_t+\gamma G_{t+1}$ and $r_t$ determined by $(s_t,a_t,s_{t+1})$,
\[
\E[G_t\mid s_t,a_t,s_{t+1}]=r_t+\gamma \Vpi(s_{t+1}),\qquad
\Var(G_t\mid s_t,a_t,s_{t+1})=\gamma^2\Var(G_{t+1}\mid s_{t+1}).
\]

\emph{Step 3: combine.}
\[
\Var(G_t\mid s_t)=\underbrace{\E_{a_t}\big[\Var_{s_{t+1}}(r_t+\gamma \Vpi(s_{t+1}))\big]+\sigpi(s_t)}_{=\,C_t(s_t)}
+\gamma^2\,\E_{a_t,s_{t+1}}\big[\Var(G_{t+1}\mid s_{t+1})\big].
\]
Taking $\E_{s_t\sim d^\pi_t}$ and using $\E_{s_t\sim d^\pi_t}\E_{a_t,s_{t+1}}[\,\cdot\,]=\E_{s_{t+1}\sim d^\pi_{t+1}}[\,\cdot\,]$ gives $\mathcal V_t=\E[C_t]+\gamma^2\mathcal V_{t+1}$; unrolling from $\mathcal V_T=0$ yields the sum. Under deterministic transitions $s_{t+1}$ is a point mass given $(s_t,a_t)$, so $\Var_{s_{t+1}}(\cdot)=0$ identically.
\end{proof}

\section{Proof of Lemma~\ref{lem:adv}}
\label{app:adv}

\begin{proof}
Write $\Vpi(s)=\sum_{a'}\pi(a'|s)\Qpi(s,a')$ and separate the $a'=a$ term:
\begin{align*}
\Api(s,a)&=\Qpi(s,a)-\pi(a|s)\Qpi(s,a)-\sum_{a'\neq a}\pi(a'|s)\Qpi(s,a')\\
&=\big(1-\pi(a|s)\big)\Qpi(s,a)-\sum_{a'\neq a}\pi(a'|s)\Qpi(s,a').
\end{align*}
For the upper bound use $\Qpi(s,a')\ge0$, dropping the sum:
\[
\Api(s,a)\le\big(1-\pi(a|s)\big)\Qpi(s,a).
\]
For the lower bound use $\Qpi(s,a')\le1$ and $\sum_{a'\neq a}\pi(a'|s)=1-\pi(a|s)$:
\[
\Api(s,a)\ge\big(1-\pi(a|s)\big)\Qpi(s,a)-\big(1-\pi(a|s)\big)=-\big(1-\pi(a|s)\big)\big(1-\Qpi(s,a)\big).
\]
Both bounds have magnitude at most $1-\pi(a|s)$ since $\Qpi(s,a)\in[0,1]$.
\end{proof}

\section{Proof of Theorem~\ref{thm:gate}}
\label{app:gate}

\begin{proof}
Let $X=\Api(s,a)$ with $a\sim\pi(\cdot|s)$; then $\E[X]=0$, $\Var(X)=\sigma^2$, $X\in[-\Vpi(s),1-\Vpi(s)]$.

\emph{(i) One-sided (Cantelli).} For any $u>0$, since $c+u>0$,
\[
\Prob(X\ge c)=\Prob(X+u\ge c+u)\le\Prob\big((X+u)^2\ge(c+u)^2\big)\le\frac{\E[(X+u)^2]}{(c+u)^2}=\frac{\sigma^2+u^2}{(c+u)^2},
\]
using Markov's inequality. Minimizing over $u$ gives $u^\star=\sigma^2/c$, at which the right side equals $\sigma^2/(\sigma^2+c^2)$.

\emph{Equality condition.} The bound is attained by the two-point law with mass $\sigma^2/(\sigma^2+c^2)$ at $c$ and mass $c^2/(\sigma^2+c^2)$ at $-\sigma^2/c$ (mean $0$, variance $\sigma^2$ by direct verification). Setting $c=1-\varepsilon$, $\sigma^2=\varepsilon(1-\varepsilon)$ recovers exactly the configuration of Proposition~\ref{prop:tight}: atoms at $1-\varepsilon$ and $-\varepsilon$ with masses $\varepsilon$ and $1-\varepsilon$.

\emph{(ii) Two-sided (Chebyshev).} $\Prob(|X|\ge c)\le\E[X^2]/c^2=\sigma^2/c^2$. (Summing~\eqref{eq:cantelli} over both tails gives $2\sigma^2/(\sigma^2+c^2)$, weaker than $\sigma^2/c^2$ whenever $c\ge\sigma$.)

\emph{(iii) Success-rate.} Since $\E[X]=0$, $\E[X^+]=\E[X^-]$. From $X\le1-\Vpi(s)$: $\E[X^+]\le1-\Vpi(s)$; from $X\ge-\Vpi(s)$: $\E[X^-]\le \Vpi(s)$. Hence $\E[X^+]\le\min\{\Vpi(s),1-\Vpi(s)\}$, and Markov's inequality on $X^+$ gives $\Prob(X\ge c)\le\E[X^+]/c$.
\end{proof}

\section{Why Bernstein's Inequality Is Vacuous Here}
\label{app:bernstein}

Applied to the single draw $\Api(s,a)$, Bernstein's inequality reads
\begin{equation}
\Prob\big(|\Api(s,a)|\ge c\big)\;\le\;2\exp\!\left(-\frac{c^2/2}{\sigma^2+c/3}\right).
\label{eq:bern}
\end{equation}
The $c^2$ in the numerator suggests exponential decay in $c$; this appendix records why no such decay is available on the admissible domain. Write $E(c,\sigma^2)=\dfrac{c^2/2}{\sigma^2+c/3}$, so that~\eqref{eq:bern} is $2e^{-E}$.

\emph{Domain.} $\Qpi\in[0,1]\Rightarrow|\Api|\le1\Rightarrow$ the event $\{|\Api|\ge c\}$ is empty for $c>1$, so only $c\in(0,1]$ is meaningful. By Bhatia--Davis applied to $\Qpi\in[0,1]$ with mean $\Vpi(s)$, $\sigma^2\le \Vpi(s)(1-\Vpi(s))\le1/4$.

\emph{Maximum of the exponent.} $E$ is increasing in $c$ and decreasing in $\sigma^2$, so $\sup E=E(1,0)=\tfrac{1/2}{1/3}=\tfrac32$, whence $2e^{-E}\ge2e^{-3/2}=0.4463\ldots$ uniformly.

\emph{Vacuity threshold.} The bound exceeds $1$ iff $E\le\log2$. At $\sigma^2=0$: $E=3c/2\le\log2\iff c\le\tfrac23\log2=0.4621$. At $\sigma^2=\tfrac14$: solving $c^2-\tfrac{2\log2}{3}c-\tfrac{\log2}{2}=0$ gives $c\le0.8635$.

\emph{Limit at low variance.} $\lim_{\sigma^2\to0}2e^{-E(c,\sigma^2)}=2e^{-3c/2}\ge2e^{-3/2}$: the bound does not vanish as $\sigpi\to0$, so it cannot certify that discovery is rare at a non-critical fork.

\emph{Comparison on the canonical fork.} On the configuration of Proposition~\ref{prop:tight} with $c=1-\varepsilon$ and $\sigma^2=\varepsilon(1-\varepsilon)$, the true probability is $\varepsilon$. Equation~\eqref{eq:cantelli} returns $\varepsilon$ exactly; \eqref{eq:bern} returns a quantity bounded below by $0.446$, an overstatement by a factor $\varepsilon^{-1}$.

\emph{Diagnosis.} The failure is structural rather than a matter of constants. Bernstein's exponential regime is driven by the sample count $n$ in $\sum_{i=1}^n X_i$; the useful exponent scales as $nc^2/(\sigma^2+bc/3)$. Here $n=1$ and $b=1$, so one is permanently pre-asymptotic and the $c/3$ term dominates the $\sigma^2$ term over most of the domain --- which is exactly why the $\sigma^2$ dependence, the object of interest, is washed out. Recovering an exponential requires reintroducing a count: over $N$ i.i.d. draws, $\Prob(\text{no discovery})=(1-p_c)^N\le e^{-Np_c}$, exponential in $N$ at rate $p_c\le\sigma^2/(\sigma^2+c^2)$. This is the statement Corollary~\ref{cor:budget} makes, with $\sigma^2$ appearing where it belongs.

\section{The Chain of Forks: Derivations}
\label{app:chain}

\subsection{Value and advantage identities}

By backward induction from $\Vpi(s_m)=1$: the wrong branch is absorbing with reward $0$, so $\Vpi(s_j)=p_j\Vpi(s_{j+1})$, giving $\Vpi(s_j)=P_j$. Then $\Qpi(s_j,\text{correct})=P_{j+1}$, $\Qpi(s_j,\text{wrong})=0$, and
\[
\Api(s_j,\text{correct})=P_{j+1}-P_j=(1-p_j)P_{j+1},\qquad
\Api(s_j,\text{wrong})=-P_j=-p_jP_{j+1}.
\]
For the policy variance, $\E[\Qpi]=P_j$ and $\E[(\Qpi)^2]=p_jP_{j+1}^2$, hence $\sigpi(s_j)=p_j(1-p_j)P_{j+1}^2$.

\emph{Budget consistency.} With $c=\Api(s_j,\text{correct})=(1-p_j)P_{j+1}$,
\[
1+\frac{c^2}{\sigpi(s_j)}=1+\frac{(1-p_j)^2P_{j+1}^2}{p_j(1-p_j)P_{j+1}^2}=\frac1{p_j},
\]
exactly the number of draws needed to sample the correct action once: Corollary~\ref{cor:budget} is attained.

\subsection{Proof of Proposition~\ref{prop:snr}}

\begin{proof}
Conditioned on $(s_j,\text{correct})$ the trajectory succeeds iff the entire suffix succeeds, so $G_j\sim\mathrm{Bernoulli}(P_{j+1})$. With the exact baseline $\Vpi(s_j)=P_j$,
\[
\E[\hat A^{\mathrm{MC}}]=P_{j+1}-P_j=\Api(s_j,\text{correct}),\qquad
\Var(\hat A^{\mathrm{MC}})=P_{j+1}(1-P_{j+1}),
\]
giving the stated SNR. Averaging $n$ i.i.d. rollouts scales the standard error by $n^{-1/2}$; setting $\sqrt n\,\mathrm{SNR}_j=1$ gives $n_j=(1-p_j)^{-2}(1-P_{j+1})/P_{j+1}\asymp(1-p_j)^{-2}P_{j+1}^{-1}$, and $P_{j+1}^{-1}=\exp(\sum_{k>j}\log(1/p_k))$. For the wrong branch, $G_j=0$ almost surely, so $\hat A^{\mathrm{MC}}=-P_j$ deterministically.
\end{proof}

\subsection{Proof of Proposition~\ref{prop:log}}

\begin{proof}
$\Vpi(s_j)=p_j\Vpi(s_{j+1})>0$; take logarithms. Each log-space regression target $\log \Vpi(s_j)-\log \Vpi(s_{j+1})=\log p_j$ is a single bounded increment, so the target is $O(1)$-conditioned at every horizon; how per-step errors propagate is then the standard TD contraction question, not a conditioning problem. In value space, an estimator with uniform absolute error $\delta$ yields a TD advantage whose bias is up to $2\delta$ against a true signal of magnitude $(1-p_j)P_{j+1}$; the estimate carries information only if $\delta\ll(1-p_j)P_{j+1}$, i.e. only under absolute accuracy exponentially small in $m-j$.
\end{proof}

\section{Fork Taxonomy: Identifiability and Gradient SNR}
\label{app:taxonomy}

\subsection{Identifiability of \texorpdfstring{$(\varepsilon,\Delta)$}{(eps,Delta)} from \texorpdfstring{$(\sigpi,H)$}{(variance,entropy)}}

For a binary fork with minority mass $\varepsilon\in(0,\tfrac12]$ and value gap $\Delta\in(0,1]$,
\[
\sigpi(s)=\varepsilon(1-\varepsilon)\Delta^2,\qquad
H(s)=-\varepsilon\log_2\varepsilon-(1-\varepsilon)\log_2(1-\varepsilon).
\]
The map $(\varepsilon,\Delta)\mapsto(\sigpi,H)$ is a bijection on $(0,\tfrac12]\times(0,1]$: $H$ is strictly increasing in $\varepsilon$ on $(0,\tfrac12]$, so $\varepsilon=H^{-1}(H(s))$, and then $\Delta=\big(\sigpi(s)/(\varepsilon(1-\varepsilon))\big)^{1/2}$. Measuring the pair --- and only the pair --- identifies the fork type; neither statistic alone suffices.

\subsection{Gradient SNR at a Type II fork}

Let the fork of Proposition~\ref{prop:tight} be softmax-parameterized with logits $(z_e,z_d)$, so $\pi(a_e)=\varepsilon$, $\pi(a_d)=1-\varepsilon$, and
\[
\frac{\partial\log\pi(a_e)}{\partial z_e}=1-\varepsilon,\qquad
\frac{\partial\log\pi(a_d)}{\partial z_e}=-\varepsilon.
\]
(The $z_d$-component of the gradient is the negation of the $z_e$-component, so the SNR below is unaffected by working with one coordinate.) With $\Api(s,a_e)=1-\varepsilon$, $\Api(s,a_d)=-\varepsilon$, the single-sample REINFORCE gradient $g=\Api(s,a)\,\partial_{z_e}\log\pi(a)$ takes value $(1-\varepsilon)^2$ w.p.\ $\varepsilon$ and $\varepsilon^2$ w.p.\ $1-\varepsilon$. Hence
\[
\E[g]=\varepsilon(1-\varepsilon)^2+(1-\varepsilon)\varepsilon^2=\varepsilon(1-\varepsilon),\qquad
\E[g^2]=\varepsilon(1-\varepsilon)^4+(1-\varepsilon)\varepsilon^4,
\]
so $\E[g]/\sqrt{\Var(g)}=\sqrt\varepsilon\,(1+O(\varepsilon))$ and $\Theta(1/\varepsilon)$ samples give unit gradient SNR. This matches Corollary~\ref{cor:budget} with $c=1-\varepsilon$, $\sigpi=\varepsilon(1-\varepsilon)$ --- budget $1/\varepsilon$ --- exactly, not merely in order. The agreement holds because score and advantage are perfectly aligned at a two-point fork, and it confirms that $\sigpi$ is the operative statistic for the gradient estimator, not merely for the return.

\end{document}